\documentclass{article}
\usepackage[utf8]{inputenc}
\usepackage{fullpage}
\usepackage{tikz}
\usepackage{subcaption}
\usepackage{hyperref}
\usepackage[shortlabels]{enumitem}
\usetikzlibrary{automata,positioning}
\usepackage{microtype}
\usepackage{graphicx}
\usepackage{booktabs}
\usepackage{xcolor}
\usepackage{hyperref}

\usepackage{amssymb,amsmath,amsthm}
\usepackage{algorithm,algorithmic}

\DeclareMathOperator*{\argmax}{arg\,max}
\DeclareMathOperator*{\poly}{poly}
\newcommand{\eh}[1]{\textsf{\color{green} Elad : #1}}

\newcommand{\sk}[1]{\textsf{\color{magenta} SK : #1}}

\title{Provably Efficient Maximum Entropy Exploration}
\author{
  Elad Hazan$^{1\,2}$ \qquad Sham M. Kakade$^{1\,3\,4}$ \qquad Karan Singh$^{1\,2}$ \qquad Abby Van Soest$^{2}$\\
  \\
  $^1$ Google AI Princeton \\
  $^2$ Department of Computer Science, Princeton University \\
  $^3$ Allen School of Computer Science and Engineering, University of Washington \\
  $^4$ Department of Statistics, University of Washington\\
  \texttt{\{ehazan,karans,asoest\}@princeton.edu}, \texttt{sham@cs.washington.edu} \\
}   
\date{}

\DeclareMathOperator*{\argmin}{argmin}

\def\K{{\mathcal K}}

\def\reals{{\mathbb R}}

\newcommand{\ignore}[1]{}

\def\reals{{\mathbb R}}

\def\bold0{\mathbf{0}}

\newcommand{\eps}{\varepsilon}

\def\eps{\varepsilon}
\def\epsilon{\varepsilon}

\newtheorem{theorem}{Theorem}[section]

\newtheorem{lemma}[theorem]{Lemma}
\newtheorem{corollary}[theorem]{Corollary}

\newcommand{\newreptheorem}[2]{%
\newenvironment{rep#1}[1]{%
 \def\rep@title{#2 \ref{##1}}%
 \begin{rep@theorem}}%
 {\end{rep@theorem}}}

\newreptheorem{theorem}{Theorem}
\newreptheorem{lemma}{Lemma}
\newreptheorem{proposition}{Proposition}
\newreptheorem{claim}{Claim}
\newreptheorem{corollary}{Corollary}
\newreptheorem{mainlemma}{Main Lemma}

\newcommand{\namedref}[2]{\mbox{\hyperref[#2]{#1~\ref*{#2}}}}

\newcommand{\figurerefb}[2]{\mbox{\hyperref[#1]{Figure~\ref*{#1}#2}}}

\newcommand{\equationref}[1]{\mbox{\hyperref[#1]{(\ref*{#1})}}}
\renewcommand{\eqref}{\equationref}

\numberwithin{equation}{section}

\begin{document}

\maketitle

\def\arxiv{}

\begin{abstract}
Suppose an agent is in a (possibly unknown) Markov Decision Process in
the absence of a reward signal, what might we hope that an agent can
efficiently learn to do? 
This work studies a broad class of objectives that are
defined solely as functions of the state-visitation frequencies that are
induced by how the agent behaves. For example, one natural, intrinsically defined, objective
problem is for the agent to learn a policy which induces a
distribution over state space that is as uniform as possible, which
can be measured in an entropic sense.  We provide an efficient
algorithm to optimize such such intrinsically defined objectives, when
given access to a black box planning oracle (which is robust to
function approximation). Furthermore, when restricted to the tabular
setting where we have sample based access to the MDP, our proposed
algorithm is provably efficient, both in terms of its sample and
computational complexities. Key to our algorithmic methodology is
utilizing the conditional gradient method (a.k.a. the Frank-Wolfe
algorithm) which utilizes an approximate MDP solver.
\end{abstract}

\ignore{
\section{TODO for ICML}

\begin{enumerate}
    \item    experiments: try out using KL rather than entropy. even for Mtn. car and cart pole. approximation by gaussian projection. Ant and humanoid (all of these are handled by Abby)
    \item make clear that we competing against all non-stationary policies.
    \item show that opt can be made stationary and stochastic.
    \item 
    add explanation that max-ent may take convex form over fully observable MDP in LP for (verify that), but not in general. 
    Justify the oracle model by DNN etc. \sk{yes, I believe i can make it convex. the oracle model is fine.}  \eh{I hope I explained this below in alternate intro}
    
    \item
    cite the Farias-van-Roy paper: for some linear approximations the above may also hold true \eh{done}
    
    \item
    experiments: probably something much more extensive is needed:
    \begin{enumerate}
        \item 
        comparators beyond random?
        \item
        objectives different than max-ent that we can handle, such as log-barrier, relative entropy etc.
        
        \item
        mujoco, other more challenging tasks. overcome the exponential grid problem \eh{currently handled by random projection, seems sufficient for this paper} 
        
        \item
        think of continuous spaces in general, how to max-ent without discretizing? 
    \end{enumerate}
\end{enumerate}
}

\section{Introduction} 
A fundamental problem in reinforcement learning is that of exploring the state space. How do we understand what is even possible in a given environment without a reward signal?  

This question has received a lot of attention, with approaches such as learning with intrinsic reward and curiosity driven methods, surveyed below. Our work studies a class of objectives that is defined solely as function of the state-visitation frequencies. A natural such objective is finding a policy that maximizes the entropy of the induced distribution over the state space. More generally, our approach extends to any concave function over distributions.

Suppose the MDP is fully and precisely known, in terms of states, actions, and the entire transition matrix. Then maximizing the entropy can be recast as a convex optimization problem (see Section~\ref{conv} or \cite{de2003linear}) over the space of state-visitation frequencies induced by the exhaustive set of all policies. However, most RL instances that are common in practice exhibit at least one of several complications:\\
---  prohibitively large state space (i.e. Chess or Go) \\
---  unknown transition matrix (as in common Atari games) \\
These scenarios often require function approximation, ie. restricting the search to a non-linearly parameterized policy class (eg. neural networks), which makes the entropy maximization problem non-convex. 

As a remedy for the computational difficulty, we propose considering an approximate planning oracle: an efficient method that given a well-specified reward signal can find an optimizing policy. Such sample-based planning oracles have been empirically observed to work well with non-linear policy classes. Given such an oracle, we give a provably efficient method for exploration based on the conditional gradient (or Frank-Wolfe) algorithm \cite{frank1956algorithm}. 

Formally, we show how to generate a sequence of reward signals, that sequentially optimized give rise to a policy with entropy close to optimal. Our main theorem gives a bound on the number of calls to the planning oracle, which is independent of the size of the state space of the MDP. 
Next, we outline the efficient construction of such oracles and state the resultant sample \& computational complexity in the tabular MDP setting. As a proof of concept, we implement our method and show experiments over several mainstream RL tasks in Section \ref{sec:experiments}. 



\subsection{Informal statement of contributions} \label{subsec:this}

To facilitate exploration in potentially unknown MDPs within a restricted policy class, we assume access to the environment using the following two oracles: 

\textbf{Approximate planning oracle:} Given a reward function
  (on states) $r:\mathcal{S}\to\reals$ and a sub-optimality gap
  $\varepsilon$,
  the planning oracle returns a stationary policy $\pi=\textsc{ApproxPlan}(r,\varepsilon)$ with
  the guarantee that $V(\pi) \geq \max_\pi V(\pi) - \varepsilon$, where $V(\pi)$ is the value of policy $\pi$. 
  
\textbf{State distribution estimate oracle:} A state
      distribution oracle estimates the state distribution
      $\hat{d}_\pi=\textsc{DensityEst}(\pi,\eps)$ of any given (non-stationary)
      policy $\pi$, guaranteeing that $\|d_{\pi}-\hat{d}_\pi\|_\infty\leq
      \varepsilon$.

\ignore{
\begin{enumerate}
\item \textbf{Approximate planning oracle:} Given a reward function
  (on states) and desired accuracy, 
  the planning oracle returns a stationary\footnote{As
  the oracle is solving a discounted problem, we know the optimal
  value is achieved by a stationary policy.} policy that is approximately optimal according to the given reward and accuracy. 
  
    \item \textbf{State distribution estimate oracle:} Given a policy and desired accuracy, this oracle estimates the state distribution
      of the given policy up to the accuracy. 
\end{enumerate}
}      

Given access to these two oracles, we describe a method that provably optimizes any continuous and smooth objective over the state-visitation frequencies. Of special interest is the maximum entropy and relative entropy objectives.

\begin{theorem}[Main Theorem - Informal] \label{main-informal} There exists an efficient algorithm (Algorithm \ref{mainA}) such that for any $\beta$-smooth measure $R$, and any $\varepsilon>0$, 
in $O(\frac{1}{\varepsilon} \log \frac{1}{\varepsilon} ) $ calls to $\textsc{ApproxPlan}$ \& \textsc{DensityEst} , it returns a policy $\bar{\pi}$  with 
\[
R(d_{\bar{\pi}}) \geq \max_\pi R(d_\pi) - \varepsilon \, .
\] 
\end{theorem}

\subsection{Related work} \label{subsec:prior}
We review related works in this section.

\textbf{Reward Shaping \& Imitation Learning:} Direct optimization approaches to RL (such as policy gradient methods)  tend to perform favorably when random sequences of actions lead the agent to some positive reward, but tend to fail when the rewards are sparse or myopic. Thus far, the most practical approaches to address this have either been through some carefully constructed reward shaping (e.g. ~\cite{Ng1999PolicyIU} where dense reward functions are provided to make the optimization problem more tractable) or through inverse reinforcement learning and imitation learning~\cite{Abbeel2004ApprenticeshipLV,Ross2011ARO} (where an expert demonstrates to the agent how to act).

\textbf{PAC Learning:} For the case of tabular Markov decision processes,  the balance of exploration and exploitation has been addressed in that there are a number of methods which utilize confidence based reward bonuses to encourage exploration in order to ultimately behave near optimally
\cite{kearns2002near,kakade2003sample,strehl2006pac, lattimore2014near, dann2015sample, szita2010model, azar2017minimax}. 

\textbf{Count-based Models \& Directed Exploration:} There are a host of recent empirical success using deep RL methods which encourage exploration in some form\cite{mnih2015human,silver2016mastering}.  
The approaches which encourage exploration are based on a few related ideas: that of encouraging encouraging exploration through state visitation frequencies (e.g. \cite{DBLP:conf/icml/OstrovskiBOM17,DBLP:conf/nips/BellemareSOSSM16,NIPS2017_6868}) and those based on a intrinsic reward signal derived from novelty or prediction error~\cite{NIPS2012_4642, pathakICMl17curiosity,DBLP:journals/corr/abs-1810-02274,NIPS2017_6851,NIPS2015_5668,NIPS2016_6591, weber2017imagination}, aligning an intrinsic reward to the target objective~\cite{Kaelbling93b,NIPS2004_2552,Singh2010IntrinsicallyMR,Singh_wheredo,Zheng2018OnLI}, or sample based approaches to tracking of value function uncertainty~\cite{NIPS2016_6501,NIPS2018_8080}.

\textbf{Intrinsic Learning:} Works in~\cite{NIPS2004_2552,Singh_wheredo,Singh2010IntrinsicallyMR} established computational theories of intrinsic reward signals (and how it might help with downstream learning of tasks) and other works also showed how to incorporate intrinsic rewards (in the absence of any true reward signal) ~\cite{DBLP:journals/corr/abs-1811-11359,RandNetDist,pathak18largescale,DBLP:journals/corr/abs-1807-04742}. The potential benefit is that such learning may help the agent reach a variety of achievable goals and do well on other extrinsically defined tasks, not just the task under which it was explicitly trained for under one specific reward function (e.g. see ~\cite{NIPS2004_2552,Singh_wheredo,DBLP:journals/corr/abs-1811-11359, DBLP:journals/corr/abs-1807-04742}).

\section{Preliminaries}
\textbf{Markov decision process:} An infinite-horizon discounted Markov Decision Process is a tuple $\mathcal{M}=(\mathcal{S}, \mathcal{A}, r, P, \gamma, d_0)$, where $\mathcal{S}$ is the set of states, $\mathcal{A}$ is the set of actions, and $d_0$ is the distribution of of the initial state $s_0$. At each timestep $t$, upon observing the state $s_t$, the execution of action $a_t$ triggers an observable reward of $r_t = r(s_t, a_t)$ and a transition to a new state $s_{t+1}\sim P(\cdot|s_t, a_t)$. The performance on an infinite sequence of states \& actions (hereafter, referred to as a \textit{trajectory}) is judged through the (discounted) cumulative reward it accumulates, defined as
\[ V(\tau=(s_0, a_0, s_1, a_1, \dots)) = (1-\gamma) \sum_{t=0}^\infty \gamma^t r(s_t, a_t). \]

\textbf{Policies:} A policy is a (randomized) mapping from a history, say $(s_0, a_0,r_0, s_1, a_1,r_1 \dots s_{t-1}, a_{t-1},r_{t-1})$, to an action $a_t$. A stationary policy $\pi$ is a (randomized) function which maps a state to an action in a time-independent manner, i.e. $\pi:\mathcal{S}\to \Delta(\mathcal{A})$. When a policy $\pi$ is executed on some MDP $\mathcal{M}$, it produces a distribution over infinite-length trajectories $\tau=(s_0, a_0, s_1, a_1 \dots )$ as specified below.
\[ P(\tau| \pi) = P(s_0) \prod_{i=0}^\infty (\pi(a_i|s_i) P(s_{i+1}|s_i, a_i)) \]
The (discounted) value $V_\pi$ of a policy $\pi$ is the expected cumulative reward an action sequence sampled from the policy $\pi$ gathers.
\[ V_\pi = \mathop{\mathbb{E}}_{\tau\sim P(\cdot|\pi)} V(\tau) = (1-\gamma)\mathop{\mathbb{E}}_{\tau\sim P(\cdot|\pi)} \sum_{t=0}^\infty \gamma^t r(s_t, a_t)\]

\ignore{While only relevant to Section~\ref{secORACONS}, define $V_\pi(s), Q_\pi(s,a)$ to be 
\begin{align}
    V_\pi(s) &= \mathop{\mathbb{E}}_{\tau\sim P(\cdot|\pi)} [V(\tau) | s_0=s]\\
    Q_\pi(s,a) &= \mathop{\mathbb{E}}_{\tau\sim P(\cdot|\pi)} [V(\tau) | s_0=s, a_0=a]
\end{align}
Also, define $V^*(s)=\max_\pi V_\pi(s)$ and $Q^*(s,a)=\max_\pi Q_\pi({s,a})$.}

\textbf{Induced state distributions:} The $t$-step state distribution and the (discounted) state distribution of a policy $\pi$ that result are
\begin{align} 
d_{t,\pi}(s) &= P(s_t=s|\pi) = \sum_{\text{all }\tau \text{ with } s_t=s} P(\tau|\pi), \\
d_{t,\pi}(s,a) &= P(s_t=s,a_t=a|\pi) = \sum_{\text{all }\tau \text{ with } s_t=s,a_t=a} P(\tau|\pi), \\
d_\pi (s) &= (1-\gamma) \sum_{t=1}^\infty \gamma^t d_{t,\pi}(s),\\
d_\pi (s,a) &= (1-\gamma) \sum_{t=1}^\infty \gamma^t d_{t,\pi}(s,a).
\end{align}
The latter distribution can be viewed as the analogue of the stationary distribution in the infinite horizon setting.

\textbf{Mixtures of stationary policies:} Given a sequence of $k$ policies $C=(\pi_0, \dots \pi_{k-1})$, and $\alpha\in \Delta_k$, we define $\pi_{\textrm{mix}}=(\alpha, C)$ to be a mixture over stationary policies. The (non-stationary) policy
$\pi_{\textrm{mix}}$ is one where, at the first timestep $t=0$, we sample policy
$\pi_i$ with probability $\alpha_i$ and then use this policy for all
subsequent timesteps.
In particular, the behavior of a mixture
$\pi_{\textrm{mix}}$ with respect to an MDP is that it induces infinite-length trajectories $\tau=(s_0, a_0, s_1, a_1 \dots )$ with the probability
law :
\begin{equation}
    P(\tau| \pi_{\textrm{mix}}) = \sum_{i=0}^{k-1} \alpha_i P(\tau|\pi_i) \label{eqMIX}
  \end{equation}  and the induced state distribution is:
\begin{equation}
d_{\pi_{\textrm{mix}}} (s) = \sum_{i=0}^{k-1} \alpha_i d_{\pi_i} (s).
\label{d_of_mix}
\end{equation}
Note that such a distribution over policies need not be representable
as a stationary stochastic policy (even if the $\pi_i$'s are
stationary) due to that the sampled actions are no
longer conditionally independent given the states.

\section{The Objective: MaxEnt Exploration}
As each policy induces a distribution over states, we can associate a \emph{concave} reward functional $R(\cdot)$ with this induced distribution. We say that a policy $\pi^*$ is a \emph{maximum-entropy exploration policy}, also to referred to as the \emph{max-ent} policy,  if the corresponding induced state distribution has the maximum possible $R(d_\pi)$ among the class of all policies. Lemma~\ref{l-mdp-stat} assures us that the search over the class of stationary policies is sufficient.
\[\pi^* \in \argmax_{\pi} R(d_\pi).\]
Our goal is to find a policy that induces a state distribution with a comparable value of the reward functional.  

\subsection{Examples of reward functionals}
A possible quantity of interest that serves as a motivation for considering such functionals is the entropy of the induced distribution.
\[
  \max_\pi \{ H(d_\pi) = - \mathop{\mathbb{E}}_{s\sim d_\pi} \log  d_\pi(s)\} \, 
\]
The same techniques we derive can 
also be used to optimize other entropic measures. For example, we may be interested in minimizing:
\[
\min_\pi \left\{\textrm{KL}(d_\pi||Q) = \mathop{\mathbb{E}}_{s\sim d_\pi} \log  \frac{d_\pi(s)}{Q(s)}\right\} \,
\]
for some given distribution $Q(s)$.
Alternatively, we may seek to minimize a cross entropy measure:
\[
\min_\pi \left\{\mathop{\mathbb{E}}_{s\sim Q} \log  \frac{1}{d_\pi(s)} = \textrm{KL}(Q||d_\pi)+H(Q)\right\} \,
\]
where the expectation is now under $Q$. For uniform $Q$, this latter measure may be more aggressive in forcing $\pi$ to have more uniform coverage than the entropy objective.

\subsection{Landscape of the objective function}
In this section, we establish that the entropy of the state distribution is \emph{not} a concave function of the policy. Similar constructions can establish analogous statements for other non-trivial functionals. Subsequently, we discuss a possible convex reformulation of the objective in the space of induced distributions which constitute a convex set.

\subsubsection{Non-convexity in the policy space}\label{nc}
Despite the concavity of the entropy function, our overall maximization problem is not concave as the state distribution is not an affine function of the policy.
This is stated precisely in the following lemma.
\begin{lemma}
$H(d_\pi)$ is not concave in $\pi$.
\end{lemma}
\begin{proof}
\begin{figure}[t!]
    \centering 
    \begin{subfigure}[b]{0.3\textwidth}
    \scalebox{0.7}{
        \begin{tikzpicture}[shorten >=1pt,node distance=2cm,on grid,auto] 
            \node[state,initial] (s_{0,0})   {$s_{0,0}$}; 
            \node[state] (s_{1,0}) [below left=of s_{0,0}] {$s_{1,0}$}; 
            \node[state] (s_{1,1}) [below right=of s_{0,0}] {$s_{1,1}$}; 
            \node[state] (s_{2,0}) [below left=of s_{1,0}] {$s_{2,0}$}; 
            \node[state] (s_{2,1}) [below right=of s_{1,0}] {$s_{2,1}$}; 
            \node[state] (s_{2,2}) [below right=of s_{1,1}] {$s_{2,2}$}; 
            \path[->] 
            (s_{0,0}) edge  node {$\frac{1}{2}$} (s_{1,0})
                  edge  node [swap] {$\frac{1}{2}$} (s_{1,1})
            (s_{1,0}) edge  node {$\frac{3}{4}$} (s_{2,0})
                  edge  node [swap] {$\frac{1}{4}$} (s_{2,1})
            (s_{1,1}) edge  node {$\frac{1}{4}$} (s_{2,1})
                  edge  node [swap] {$\frac{3}{4}$} (s_{2,2})
            (s_{2,0}) edge [loop below] node {$1$} (s_{2,0})
            (s_{2,1}) edge [loop below] node {$1$} (s_{2,1})
            (s_{2,2}) edge [loop below] node {$1$} (s_{2,2});
        \end{tikzpicture}}
        \caption{$\pi_0$, $d_{2,\pi_0}=\left(\frac{3}{8},\frac{1}{4},\frac{3}{8}\right)$}
    \end{subfigure}
    \begin{subfigure}[b]{0.3\textwidth}
    \scalebox{0.7}{
        \begin{tikzpicture}[shorten >=1pt,node distance=2cm,on grid,auto] 
            \node[state,initial] (s_{0,0})   {$s_{0,0}$}; 
            \node[state] (s_{1,0}) [below left=of s_{0,0}] {$s_{1,0}$}; 
            \node[state] (s_{1,1}) [below right=of s_{0,0}] {$s_{1,1}$}; 
            \node[state] (s_{2,0}) [below left=of s_{1,0}] {$s_{2,0}$}; 
            \node[state] (s_{2,1}) [below right=of s_{1,0}] {$s_{2,1}$}; 
            \node[state] (s_{2,2}) [below right=of s_{1,1}] {$s_{2,2}$}; 
            \path[->] 
            (s_{0,0}) edge  node {$\frac{2}{3}$} (s_{1,0})
                  edge  node [swap] {$\frac{1}{3}$} (s_{1,1})
            (s_{1,0}) edge  node {$\frac{1}{2}$} (s_{2,0})
                  edge  node [swap] {$\frac{1}{2}$} (s_{2,1})
            (s_{1,1}) edge  node {$0$} (s_{2,1})
                  edge  node [swap] {$1$} (s_{2,2})
            (s_{2,0}) edge [loop below] node {$1$} (s_{2,0})
            (s_{2,1}) edge [loop below] node {$1$} (s_{2,1})
            (s_{2,2}) edge [loop below] node {$1$} (s_{2,2});
        \end{tikzpicture}}
        \caption{$\pi_1$, $d_{2,\pi_1}=\left(\frac{1}{3},\frac{1}{3},\frac{1}{3}\right)$}
    \end{subfigure}
    \begin{subfigure}[b]{0.3\textwidth}
    \scalebox{0.7}{
        \begin{tikzpicture}[shorten >=1pt,node distance=2cm,on grid,auto] 
            \node[state,initial] (s_{0,0})   {$s_{0,0}$}; 
            \node[state] (s_{1,0}) [below left=of s_{0,0}] {$s_{1,0}$}; 
            \node[state] (s_{1,1}) [below right=of s_{0,0}] {$s_{1,1}$}; 
            \node[state] (s_{2,0}) [below left=of s_{1,0}] {$s_{2,0}$}; 
            \node[state] (s_{2,1}) [below right=of s_{1,0}] {$s_{2,1}$}; 
            \node[state] (s_{2,2}) [below right=of s_{1,1}] {$s_{2,2}$}; 
            \path[->] 
            (s_{0,0}) edge  node {$\frac{1}{3}$} (s_{1,0})
                  edge  node [swap] {$\frac{2}{3}$} (s_{1,1})
            (s_{1,0}) edge  node {$1$} (s_{2,0})
                  edge  node [swap] {$0$} (s_{2,1})
            (s_{1,1}) edge  node {$\frac{1}{2}$} (s_{2,1})
                  edge  node [swap] {$\frac{1}{2}$} (s_{2,2})
            (s_{2,0}) edge [loop below] node {$1$} (s_{2,0})
            (s_{2,1}) edge [loop below] node {$1$} (s_{2,1})
            (s_{2,2}) edge [loop below] node {$1$} (s_{2,2});
        \end{tikzpicture}}
        \caption{$\pi_2$, $d_{2,\pi_2}=\left(\frac{1}{3},\frac{1}{3},\frac{1}{3}\right)$}
    \end{subfigure}
    \caption{Description of $\pi_0,\pi_1,\pi_2$.\label{fig1}}
\end{figure}
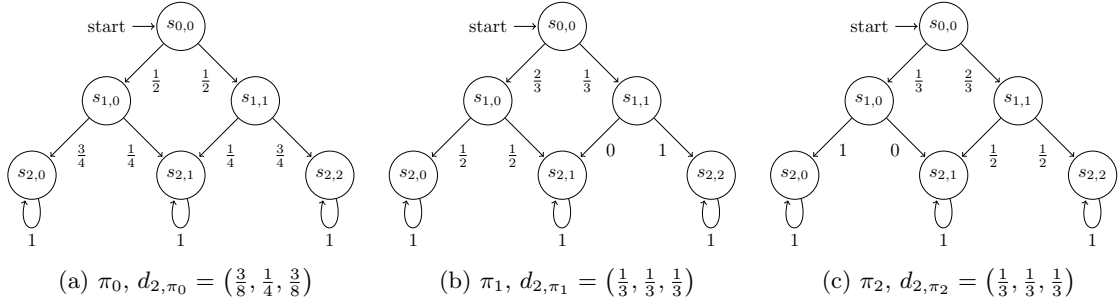
Figure~\ref{fig1} demonstrates the behavior of $\pi_0, \pi_1,\pi_2$ on a 6-state MDP with binary actions. Note that for sufficiently large $\gamma\to 1$ and any policy $\pi$, the discounted state distribution converges to the distribution on the states at the second timestep, or formally $d_{\pi}\to d_{2,\pi}$. Now with the realization $\pi_0=\frac{\pi_1+\pi_2}{2}$, observe that $d_{2,\pi_0}$ is not uniform on $\{s_{2,0},s_{2,1}, s_{2,2}\}$, implying that $H(d_{2,\pi_0})  < \frac{H(d_{2,\pi_1})+H(d_{2,\pi_2})}{2}$. 
\end{proof}

\begin{lemma}\label{lemDPI}
For any policy $\pi$ and MDP $\mathcal{M}$, define the matrix $P_\pi\in\reals^{|\mathcal{S}|\times |\mathcal{S}|}$ so that 
\[P_\pi(s',s) = \sum_{a\in \mathcal{A}} \pi(a|s)P(s'|s,a).\]
Then it is true that
\begin{enumerate}
    \item $P_\pi$ is linear in $\pi$,
    \item $d_{t,\pi} = P_\pi^t d_0$ for all $t\geq 0$,
    \item $d_\pi = (1-\gamma) (I-\gamma P_\pi)^{-1} d_0$.
\end{enumerate}
\end{lemma}
\begin{proof}
Linearity of $P_\pi$ is evident from the definition. (2,3) may be verified by calculation.
\end{proof}

\subsubsection{Convexity in the distribution space}\label{conv}
Define the set of all induced distributions as $\K=\{d: d(s,a)\geq 0\text{ and satisfies the constraints stated below}\}$. For every $d\in \K$, it is possible to construct a policy $\pi$ with $d_\pi=d$, and for every $\pi$, $d_\pi\in \K$ holds \cite{puterman2014markov}.
\begin{align*}
    \sum_{a} d(s,a) &= (1-\gamma)d_0(s) + \gamma \sum_{s',a'} P(s|s',a') d(s',a')  
\end{align*}
The search for a max-ent policy can be recast as a convex optimization problem over the space of distributions.
\[ \max_{d\in \K} R(d). \]
Although we outline this reduction for an unrestricted policy class, similar reductions are possible for linearly-parameterized policy classes. These techniques can be extended to the case of MDPs with unknown dynamics \cite{de2003linear}.

\subsubsection{Sufficiency of Stationary Policies}
The set of non-Markovian policies is \emph{richer} than the set of Markov stationary policies in terms of the distributions over trajectories each may induce. A priori, it is not evident that maximizing $R(d_\pi)$ over the set of stationary policies is sufficient to guarantee the optimality in a larger class of all policies. Lemma~\ref{l-mdp-stat} establishes this claim by equating the set of achievable \emph{induced state distributions} for these two sets of policies.

\begin{lemma}\label{l-mdp-stat}\cite{puterman2014markov}
For any possibly non-Markovian policy $\pi$, define a stationary Markov policy $\pi'$ as $\pi'(a|s) = \frac{d_{\pi}(s,a)}{d_{\pi}(s)}$. Then, $d_{\pi}=d_{\pi'}$.
\end{lemma}

\section{Algorithms \& Main Results} \label{sec:alg-main}
The algorithm maintains a distribution over policies, and proceeds by
adding a new policy to the support of the mixture and reweighing the
components. To describe the algorithm, we will utilize access to two kinds of oracles. The constructions for these are detailed in later sections.

\textbf{Approximate planning oracle:} Given a reward function
  (on states) $r:\mathcal{S}\to\reals$ and a sub-optimality gap
  $\varepsilon_1$,
  the planning oracle returns a stationary\footnote{As
  the oracle is solving a discounted problem, we know the optimal
  value is achieved by a stationary policy.} policy $\pi=\textsc{ApproxPlan}(r,\varepsilon_1)$ with
  the guarantee that $V_\pi \geq \max_\pi V_\pi - \varepsilon_1$.
  
\textbf{State distribution estimate oracle:} A state
      distribution oracle estimates the state distribution
      $\hat{d}_\pi=\textsc{DensityEst}(\pi,\eps_0)$ of any given (non-stationary)
      policy $\pi$, guaranteeing that $\|d_{\pi}-\hat{d}_\pi\|_\infty\leq
      \varepsilon_0$.  

\begin{algorithm}[t!]
\caption{Maximum-entropy policy computation.}
\label{mainA}
\begin{algorithmic}[1]
\STATE \textbf{Input:} Step size $\eta$, number
of iterations $T$, planning oracle error tolerance $\varepsilon_1>0$,
state distribution oracle error tolerance $\varepsilon_0>0$, reward functional $R$. 
\STATE Set $C_0=\{\pi_0\}$ where $\pi_0$ is an arbitrary policy.
\STATE Set $\alpha_0=1$. 
\FOR{$t = 0, \ldots, T-1$}
\STATE Call the state distribution oracle on
$\pi_{\textrm{mix},t}=(\alpha_t, C_t)$:
\[  \hat{d}_{\pi_{\textrm{mix},t}}=\textsc{DensityEst}\left(\pi_{\textrm{mix},t},\varepsilon_0\right)\]
\vspace{-0.5cm}
\STATE Define the reward function $r_t$ as 
\[ r_t(s) = \nabla R(\hat{d}_{\pi_{\textrm{mix},t}}) := \frac{dR(X)}{dX}\Bigg\vert_{X=\hat{d}_{\pi_{\textrm{mix},t}}}.\]
\vspace{-0.5cm}
\STATE Compute the (approximately) optimal policy on $r_t$:
\[
  \pi_{t+1} = \textsc{ApproxPlan}\left(r_t, \varepsilon_1 \right)
  \, .
\]
\vspace{-0.5cm}
\STATE Update $\pi_{\textrm{mix},t+1}=(\alpha_{t+1}, C_{t+1})$ to be
\begin{align}
C_{t+1} &= (\pi_0, \dots, \pi_t,\pi_{t+1}), \\
\alpha_{t+1} &= ((1 - \eta) \alpha_{t},\eta).
\end{align}
\ENDFOR
\RETURN $\pi_{\textrm{mix},T} = (\alpha_T, C_T)$. 
\end{algorithmic}
\end{algorithm}

We shall assume in the following discussion that the reward functional $R$ is $\beta$-smooth, $B$-bounded, and that it satisfies the following inequality for all $X,Y$.
\begin{align}\label{asst}
    &\|\nabla R(X)-\nabla R(Y)\|_\infty \leq \beta \|X-Y\|_\infty \\
    &-\beta \mathbb{I} \preceq  \nabla^2 R(X) \preceq \beta \mathbb{I}; \quad \|\nabla R(X)\|_\infty \leq B
\end{align} 

\begin{theorem}[Main Theorem] \label{mainT} For any $\varepsilon>0$, set
  $\varepsilon_1=0.1\varepsilon$,
  $\varepsilon_0 = 0.1\beta^{-1}\varepsilon$, and
  $\eta = 0.1\beta^{-1}\varepsilon$.  When
  Algorithm~\ref{mainA} is run for $T$ iterations where:
  \[
T \geq 10 \beta \varepsilon^{-1} \log 10 B \varepsilon^{-1} \, ,
\]
we have that:
\[
R(\pi_{\textrm{mix},T}) \geq \max_\pi R(d_\pi) - \varepsilon \, .
\] 
\end{theorem}

Before we begin the proof, we state the implication for maximizing the entropy of the induced distribution. While the entropy objective is, strictly speaking, not smooth, one may consider a smoothed alternative $H_\sigma$ defined below. 
\[ H_\sigma(d_\pi) = -\mathbb{E}_{s\sim d_\pi} \log (d_\pi(s)+\sigma)\]
When the algorithm is fed $H_\sigma$ as the \emph{proxy} reward functional, it is possible make sub-optimality guarantees on the true objective $H$. Lemma~\ref{lemBREG} (D) relates the entropy functional $H$ to its smoothed variant $H_\sigma$, while the rest of the lemma quantifies smoothness of $H_\sigma$. The factors of $|\mathcal{S}|$ incurred below are a consequence of imposed smoothing on $H$, and are not necessary for naturally smooth objectives.

\begin{corollary} \label{mainT22} For any $\varepsilon>0$, set
  $\sigma=\frac{0.1\varepsilon}{2|\mathcal{S}|}$,
  $\varepsilon_1=0.1\varepsilon$,
  $\varepsilon_0 = \frac{0.1 \varepsilon^2}{80 |\mathcal{S}|}$, and
  $\eta = \frac{0.1 \varepsilon^2}{40 |\mathcal{S}|}$.  When
  Algorithm~\ref{mainA} is run for $T$ iterations with the reward functional $H_\sigma$, where:
  \[
T \geq \frac{40 |\mathcal{S}|}{0.1\varepsilon^2} \log \frac{\log |\mathcal{S}|}{0.1\varepsilon} \, ,
\]
we have that:
\[
H(\pi_{\textrm{mix},T}) \geq \max_\pi H(d_\pi) - \varepsilon \, .
\] 
\end{corollary}

We continue with the proof of the main theorem. 

\begin{proof}[Proof of Theorem~\ref{mainT}]
Let $\pi^*$ be a maximum-entropy policy, ie. $\pi^* \in \argmax_\pi
R(d_\pi)$.
\ifdefined\arxiv
\begin{align*}
&R(d_{\pi_{\textrm{mix},t+1}}) = R((1-\eta)d_{\pi_{\textrm{mix},t}}+\eta d_{\pi_{t+1}})& \text{Equation}~\ref{d_of_mix}\\
&\geq R(d_{\pi_{\textrm{mix},t}})+\eta \langle d_{\pi_{t+1}}-d_{\pi_{\textrm{mix},t}}, \nabla R(d_{\pi_{\textrm{mix},t}})\rangle - \eta^2\beta \|d_{\pi_{t+1}}-d_{\pi_{\textrm{mix},t}}\|_2^2 & \text{smoothness}
\end{align*}
\else
\begin{align*}
&R(d_{\pi_{\textrm{mix},t+1}}) = R((1-\eta)d_{\pi_{\textrm{mix},t}}+\eta d_{\pi_{t+1}})& \text{Equation}~\ref{d_of_mix}\\
\geq& R(d_{\pi_{\textrm{mix},t}})+\eta \langle d_{\pi_{t+1}}-d_{\pi_{\textrm{mix},t}}, \nabla R(d_{\pi_{\textrm{mix},t}})\rangle \\
&- \eta^2\beta \|d_{\pi_{t+1}}-d_{\pi_{\textrm{mix},t}}\|_2^2 & \text{smoothness}
\end{align*}
\fi
The second inequality follows from the smoothness of $R$. (See Section 2.1 in \cite{bubeck2015convex} for equivalent definitions of smoothness in terms of the function value and the Hessian.)

To incorporate the error due to the two oracles, observe
\ifdefined\arxiv
\begin{align*}
    \langle d_{\pi_{t+1}}, \nabla R(d_{\pi_{\textrm{mix},t}})\rangle &\geq \langle d_{\pi_{t+1}}, \nabla R(\hat{d}_{\pi_{\textrm{mix},t}})\rangle - \beta \|d_{\pi_{\textrm{mix},t}}-\hat{d}_{\pi_{\textrm{mix},t}}\|_\infty \\
    &\geq \langle d_{\pi^*}, \nabla R(\hat{d}_{\pi_{\textrm{mix},t}})\rangle - \beta\varepsilon_0 -\varepsilon_1 \\    &
    \geq \langle d_{\pi^*}, \nabla R(d_{\pi_{\textrm{mix},t}})\rangle - 2\beta\varepsilon_0 -\varepsilon_1 
\end{align*}
\else
\begin{align*}
    &\langle d_{\pi_{t+1}}, \nabla R(d_{\pi_{\textrm{mix},t}})\rangle \\
    &\geq \langle d_{\pi_{t+1}}, \nabla R(\hat{d}_{\pi_{\textrm{mix},t}})\rangle - \beta \|d_{\pi_{\textrm{mix},t}}-\hat{d}_{\pi_{\textrm{mix},t}}\|_\infty \\
    &\geq \langle d_{\pi^*}, \nabla R(\hat{d}_{\pi_{\textrm{mix},t}})\rangle - \beta\varepsilon_0 -\varepsilon_1 \\    &
    \geq \langle d_{\pi^*}, \nabla R(d_{\pi_{\textrm{mix},t}})\rangle - 2\beta\varepsilon_0 -\varepsilon_1 
\end{align*}
\fi
The first and last inequalities invoke the assumptions laid out in Equation~\ref{asst}. Note that the second inequality above follows from the defining character of the planning oracle, ie. with respect to the reward vector $r_t=\nabla R(\hat{d}_{\pi_{\textrm{mix},t}})$, for any policy $\pi'$, it holds true that 
\[ V_{\pi_{t+1}} = \langle d_{\pi_{t+1}}, r_t\rangle \geq V_{\pi'} -\varepsilon_1 = \langle d_{\pi'}, r_t\rangle -\varepsilon_1 \]
In particular, this statement holds for the choice $\pi'=\pi^*$. This argument does not rely on $\pi^*$ being a stationary policy, since $\pi_{t+1}$ is an optimal policy for the reward function $r_{t}$ among the class of all policies.

Using the above fact and continuing on
\ifdefined\arxiv
\begin{align*}
    R(d_{\pi_{\textrm{mix},t+1}}) &\geq R(d_{\pi_{\textrm{mix},t}})+\eta \langle d_{\pi^*}-d_{\pi_{\textrm{mix},t}}, \nabla R(d_{\pi_{\textrm{mix},t}})\rangle - 2\eta\beta\varepsilon_0 -\eta\varepsilon_1- \eta^2\beta\\
    &\geq (1-\eta)R(d_{\pi_{\textrm{mix},t}})+\eta R(d_{\pi^*})- 2\eta\beta\varepsilon_0 -\eta\varepsilon_1- \eta^2\beta
\end{align*}
\else
\begin{align*}
    &R(d_{\pi_{\textrm{mix},t+1}}) \\
    \geq& R(d_{\pi_{\textrm{mix},t}})+\eta \langle d_{\pi^*}-d_{\pi_{\textrm{mix},t}}, \nabla R(d_{\pi_{\textrm{mix},t}})\rangle \\
    &- 2\eta\beta\varepsilon_0 -\eta\varepsilon_1- \eta^2\beta\\
    \geq& (1-\eta)R(d_{\pi_{\textrm{mix},t}})+\eta R(d_{\pi^*}) \\
    &- 2\eta\beta\varepsilon_0 -\eta\varepsilon_1- \eta^2\beta
\end{align*}
\fi
The last step here utilizes the concavity of $R$. Indeed, the inequality follows immediately from the sub-gradient characterization of concave functions. Now, with the aid of the above, we observe the following inequality.
\ifdefined\arxiv
\begin{align*}
& R(d_{\pi^*})-R(d_{\pi_{\textrm{mix},t+1}}) \leq (1-\eta) (R(d_{\pi^*})-R(d_{\pi_{\textrm{mix},t}})) + 2\eta\beta\varepsilon_0 +\eta\varepsilon_1+ \eta^2\beta .
\end{align*}
\else
\begin{align*}
& R(d_{\pi^*})-R(d_{\pi_{\textrm{mix},t+1}}) \\
&\leq (1-\eta) (R(d_{\pi^*})-R(d_{\pi_{\textrm{mix},t}})) + 2\eta\beta\varepsilon_0 +\eta\varepsilon_1+ \eta^2\beta .
\end{align*}
\fi
Telescoping the inequality, this simplifies to 
\ifdefined\arxiv
\begin{align*}
R(d_{\pi^*})-R(d_{\pi_{\textrm{mix},T}}) & \leq (1-\eta)^T (R(d_{\pi^*})-R(d_{\pi_{\textrm{mix},0}})) + 2\beta\varepsilon_0 +\varepsilon_1+ \eta\beta \\
&\leq B e^{-T\eta} + 2\beta\varepsilon_0 +\varepsilon_1+ \eta\beta .
\end{align*}
\else
\begin{align*}
&R(d_{\pi^*})-R(d_{\pi_{\textrm{mix},T}}) \\
& \leq (1-\eta)^T (R(d_{\pi^*})-R(d_{\pi_{\textrm{mix},0}})) + 2\beta\varepsilon_0 +\varepsilon_1+ \eta\beta \\
&\leq B e^{-T\eta} + 2\beta\varepsilon_0 +\varepsilon_1+ \eta\beta .
\end{align*}
\fi
Setting $\varepsilon_1=0.1\varepsilon$, $\varepsilon_0 = 0.1 \beta^{-1}\varepsilon$, $\eta = 0.1\beta^{-1}\varepsilon$, $T= \eta^{-1} \log 10B\varepsilon^{-1}$ suffices.
\end{proof}

The following lemma is helpful in proving the corollary for the entropy functional. 
\begin{lemma}\label{lemBREG}\label{lemGRAD}
For any two distributions $P,Q\in\Delta_d$:
\begin{enumerate}[(A)]
    \item $(\nabla H_\sigma(P))_i = -\left(\log (P_{i}+\sigma)+\frac{P_{i}}{P_{i}+\sigma}\right)$,
    \item $H_\sigma(P)$ in concave in $P$,
    \item $H_\sigma(P)$ is $2\sigma^{-1}$ smooth, ie. \[-2\sigma^{-1} \mathbb{I}_d \preceq \nabla^2 R(P) \preceq 2\sigma^{-1} \mathbb{I}_d,\]
    \item $|H_\sigma(P)-R(P)| \leq d\sigma$,
    \item $\|\nabla H_\sigma(P)-\nabla H_\sigma(Q)\|_\infty \leq 2\sigma^{-1}\|P-Q\|_\infty$.
\end{enumerate}
\end{lemma}

\begin{proof}[Proof of Lemma~\ref{lemGRAD}]
(A) may be verified by explicit calculation. Observe $\nabla^2 H_\sigma(P)$ is a diagonal matrix with entries
\[ (\nabla^2 H_\sigma(P))_{i,i} = -\frac{P_i+2\sigma}{(P_i+\sigma)^2}.\]
(B) is immediate. (C) follows as $|(\nabla^2 H_\sigma(P))_{i,i}|\leq 2\sigma^{-1}$. 
\[ |H_\sigma(P)-H(P)| = \sum_{i=0}^{d-1} P_i \log \frac{P_i+\sigma}{P_i} \leq \sum_{i=0}^{d-1} P_i \frac{\sigma}{P_i} =d\sigma. \]
The last inequality follows from $\log x\leq x-1,\forall x>0$. Finally, to see (E), using Taylor's theorem, observe 
\ifdefined\arxiv
\begin{align*}
    \|\nabla H_\sigma(P)-\nabla H_\sigma(Q)\|_\infty &\leq \max_{i,\alpha\in [0,1]} |(\nabla^2 H_\sigma(\alpha P+(1-\alpha) Q)_{i,i}| \|P-Q\|_\infty \\
    &\leq 2\sigma^{-1}\|P-Q\|_\infty.
\end{align*} 
\else
\begin{align*}
    &\|\nabla H_\sigma(P)-\nabla H_\sigma(Q)\|_\infty \\
    &\leq \max_{i,\alpha\in [0,1]} |(\nabla^2 H_\sigma(\alpha P+(1-\alpha) Q)_{i,i}| \|P-Q\|_\infty \\
    &\leq 2\sigma^{-1}\|P-Q\|_\infty.
\end{align*} 
\fi
\end{proof}

\subsection{Tabular setting}
In general, the construction of provably computationally efficient approximate planning oracle for MDPs with large or continuous state spaces poses a challenge. Discounting limited settings (eg. the Linear Quadratic Regulators \cite{bertsekas2005dynamic}, \cite{fazel2018global}), one may only appeal to the recent empirical successes of sample-based planning algorithms that rely on the power of non-linear function approximation.

Nevertheless, one may expect, and possibly require, that any solution proposed to address the general case performs reasonably when restricted to the tabular setting. In this spirit, we outline the construction of the required oracles in the tabular setting. 

\subsubsection{The known MDP case}\label{know}
With the knowledge of the transition matrix $P$ of a MDP $\mathcal{M}$ in the form of an explicit tensor, the planning oracle can be implemented via any of the exact solution methods \cite{bertsekas2005dynamic}, eg. value iteration, linear programming. The state distribution oracle can be efficiently implemented as Lemma~\ref{lemDPI} suggests.
\begin{corollary}
When the MDP $\mathcal{M}$ is known explicitly, with the oracles described in Section~\ref{sec:alg-main}, Algorithm~\ref{mainA} runs in $\poly \left(\beta, |\mathcal{S}|,|\mathcal{A}|,\frac{1}{1-\gamma}, \frac{1}{\varepsilon}, \log B\right)$ time to guarantee $R(d_{\pi_{\textrm{mix},T}}) \geq \max_\pi R(d_\pi) -\varepsilon$.
\end{corollary}

\subsubsection{The unknown MDP case}\label{not-know}
For the case of an unknown MDP, a sample-based algorithm must iteratively try to learn about the MDP through its interactions with the environment. Here, we assume a $\gamma$-discounted episodic setting, where the agent can act in the environment starting from $s_0 \sim d_0$ for some number of steps, and is then able to reset. Our measure of sample complexity in this setting is the number of $\tilde{O}\left((1-\gamma)^{-1}\right)$-length episodes the agent must sample to achieve a $\varepsilon$-suboptimal performance guarantee. 

The algorithm outlined below makes a distinction between the set of states it is (relatively) sure about and the set of states that have not been visited enough number of times yet. The algorithm and the analysis is similar to the $E^3$ algorithm \cite{kearns2002near}. Since algorithms like $E^3$ proceed by building a relatively accurate model on the set of \emph{reachable} states, as opposed to estimate of the value functions, this permits the reuse of information across different invocations, each of which might operate on a different reward signal.

\begin{theorem}\label{mainB}
For an unknown MDP, with Algorithm~\ref{orcA} as the planning oracle and Algorithm~\ref{orcB} as the distribution estimate oracle, Algorithm~\ref{mainA} runs in $\poly \left(\beta, |\mathcal{S}|,|\mathcal{A}|,\frac{1}{1-\gamma}, \frac{1}{\varepsilon}\right)$ time and executes $\tilde{O}\left(\frac{B^3|\mathcal{S}|^2 |\mathcal{A}|}{\varepsilon^3(1-\gamma)^2}+\frac{\beta^3}{\varepsilon^3}\right)$ episodes of length $\tilde{O}\left(\frac{\log |\mathcal{S}|\varepsilon^{-1}}{\log \gamma^{-1}}\right)$ to guarantee that \[R(d_{\pi_{\textrm{mix},T}}) \geq \max_\pi R(d_\pi) -\varepsilon.\]
\end{theorem}

A sub-optimality bound may be derived on the non-smooth entropy functional $H$ via Lemma~\ref{lemGRAD}. Again, the extraneous factors introduced in the process are a consequence of the imposed smoothing via $H_\sigma$.

\begin{corollary}\label{mainBB}
For an unknown MDP, with Algorithm~\ref{orcA} as the planning oracle and Algorithm~\ref{orcB} as the distribution estimate oracle and $H_\sigma$ as the \emph{proxy} reward functional, Algorithm~\ref{mainA} runs in $\poly \left(|\mathcal{S}|,|\mathcal{A}|,\frac{1}{1-\gamma}, \frac{1}{\varepsilon}\right)$ time and executes $\tilde{O}\left(\frac{|\mathcal{S}|^2 |\mathcal{A}|}{\varepsilon^3(1-\gamma)^2}+\frac{|\mathcal{S}|^3}{\varepsilon^6}\right)$ episodes of length $\tilde{O}\left(\frac{\log |\mathcal{S}|\varepsilon^{-1}}{\log \gamma^{-1}}\right)$ to guarantee that \[H(d_{\pi_{\textrm{mix},T}}) \geq \max_\pi H(d_\pi) -\varepsilon.\]
\end{corollary}

\begin{algorithm}[t!]
\caption{Sample-based planning for an unknown MDP.}
\label{orcA}
\begin{algorithmic}[1]
\STATE \textbf{Input:} Reward $r$, error tolerance $\varepsilon>0$, exact planning oracle tolerance $\varepsilon_1>0$, oversampling parameter $m$, number of rollouts $n$, rollout length $t_0$.
\STATE Initialize a persistent data structure $C\in\reals^{|\mathcal{S}|^2\times |\mathcal{A}|}$, which is maintained across different calls to the planning algorithm to keep transition counts, to $C(s'|s,a)=0$ for every $(s',s,a)\in\mathcal{S}^2\times\mathcal{A}$.
\REPEAT
\STATE Declare $\mathcal{K}=\{s: \min_{a\in\mathcal{A}} \sum_{s'\in\mathcal{S}} C(s'|s,a) \geq m\}$, $\hat{P}(s'|s,a)=\begin{cases}\frac{C(s'|s,a)}{\sum_{s'\in\mathcal{S}} C(s'|s,a)}, & \text{if }s\in \mathcal{K} \\ \mathbf{1}_{s'=s}. & \text{otherwise}.\end{cases}$
\STATE Define the reward function as $r_\mathcal{K}(s)= \begin{cases} r(s), & \text{if }s\in\mathcal{K}\\ B. & \text{otherwise}\end{cases}$.
\STATE Compute an (approximately) optimal policy on the MDP induced by $\hat{P}$ and reward $r_\mathcal{K}$. This task is purely computational, and can be done as indicated in Section~\ref{know}. Also, modify the policy so that on every state $s\in \mathcal{S}-\mathcal{K}$, it chooses the least performed action.
\[ 
\pi(s) = \begin{cases}(\Pi\left(r_\mathcal{K}, \varepsilon_1 \right))(s) & \text{if }s\in\mathcal{K}, \\ \argmin_{a\in\mathcal{A}} \sum_{s'\in\mathcal{S}}C(s'|s,a) & \text{otherwise}\end{cases}\label{consPI}
\]
\STATE Run $\pi$ on the true MDP $\mathcal{M}$ to obtain $n$ independently sampled $t_0$-length trajectories $(\tau_1,\dots \tau_n)$, and increment the corresponding counts in $C(s'|s,a)$.
\STATE If and only if no trajectory $\tau_i$ contains a state $s\in \mathcal{S}-\mathcal{K}$, mark $\pi$ as \emph{stable}. \label{cert}
\UNTIL{$\pi$ is \emph{stable}.}
\RETURN $\pi$. 
\end{algorithmic}
\end{algorithm}

\begin{algorithm}[t!]
\caption{Sample-based estimate of the state distribution.}
\label{orcB}
\begin{algorithmic}[1]
\STATE \textbf{Input:} A policy $\pi$, termination length $t_0$, oversampling parameter $m$.
\STATE Sample $m$ trajectories $(\tau_0, \dots \tau_{m-1})$ of length $t_0$ following the policy $\pi$.
\STATE For every $t< t_0$, calculate the empirical state distribution $\hat{d}_{t,\pi}$. \[d_{t,\pi}(s) = \frac{|\{i< m: \tau_i=(s_0, a_0, \dots) \text{ with } s_t=s \}|}{m}\]
\RETURN $\hat{d}_\pi = \frac{1-\gamma}{1-\gamma^{t_0}}\sum_{t=0}^{t_0-1} \gamma^t \hat{d}_{t,\pi} $
\end{algorithmic}
\end{algorithm}

Before we state the proof, we note the following lemmas. The first is an adaptation of the analysis of the $E^3$ algorithm. The second is standard. We only include the second for completeness. The proofs of these may be found in the appendix.

\begin{lemma} \label{lorcA}
For any reward function $r$ with $\|r\|_\infty \leq B$, $\varepsilon>0$, with $\varepsilon_1=0.1B^{-1}\varepsilon, m=\frac{32 B^2 |\mathcal{S}|\log \frac{2|\mathcal{S}|}{\delta}}{(1-\gamma)^2(0.1\varepsilon)^2}, n=\frac{B\log \frac{32|\mathcal{S}|^2 |\mathcal{A}|\log \frac{2|\mathcal{S}|}{\delta}}{(1-\gamma)^2 (0.1\varepsilon)^2 \delta}}{0.1\varepsilon},t_0=\frac{\log \frac{0.1\varepsilon}{\log |\mathcal{S}|}}{\log \gamma}$, Algorithm~\ref{mainB} guarantees with probability $1-\delta$
\[ V_\pi \geq \max_\pi V_\pi -\varepsilon. \]
Furthermore, note that if Algorithm~\ref{mainB} is invoked $T$ times (on possibly different reward functions), the total number of episodes sampled across all the invocations is $n(T+m|\mathcal{S}||\mathcal{A}|)=\tilde{O}\left(\frac{BT}{\varepsilon}+ \frac{B^3|\mathcal{S}|^2 |\mathcal{A}|}{\varepsilon^3(1-\gamma)^2}\right)$, each episode being of length $t_0$.
\end{lemma}

\begin{lemma}\label{lorcB}
For any $\varepsilon_0, \delta>0$, when Algorithm~\ref{orcB} is run with $m=\frac{200}{\varepsilon_0^2}\log \frac{2|\mathcal{S}|\log 0.1\varepsilon}{\delta \log \gamma}$, $t_0=\frac{\log 0.1\varepsilon_0}{\log \gamma}$, $\hat{d}_\pi$ satisfies $\|\hat{d}_\pi-d_\pi\|_\infty \leq \varepsilon_0 $ with probability at least $1-\delta$. In this process, the algorithm samples $m$ episodes of length $t_0$.
\end{lemma}

\begin{proof}[Proof of Theorem~\ref{mainB}]
The claim follows immediately from the invocations of the two lemmas above with the parameter settings proposed in Theorem~\ref{mainT}.
\end{proof}


The following notions \& lemmas are helpful in proving Lemma~\ref{lorcA}. We shall call a state $s\in\mathcal{K}$ $m$-\emph{known} if, for all actions $a\in\mathcal{A}$, action $a$ has been executed at state $s$ at least $m$ times. For any MDP $\mathcal{M}=(\mathcal{S},\mathcal{A}, r,P,\gamma)$ and a set of $m$-known states $\mathcal{K}\subseteq \mathcal{S}$, define an induced MDP $\mathcal{M}_{\mathcal{K}}=(\mathcal{S},\mathcal{A}, r_\mathcal{K},P_\mathcal{K},\gamma)$ so that the states absent from $\mathcal{K}$ are absorbing and maximally rewarding.
\begin{align}
    r_\mathcal{K}(s,a) &= \begin{cases}r(s,a) & \text{if } s\in\mathcal{K}, \\ B& \text{otherwise},\end{cases} \\
    P_\mathcal{K}(s'|s,a) &= \begin{cases}P(s'|s,a) & \text{if } s\in\mathcal{K}, \\ \mathbf{1}_{s'=s} & \text{otherwise}.\end{cases}
\end{align}
The state distribution induced by a policy $\pi$ on $\mathcal{M}_\mathcal{K}$ shall be denoted by $d_{\mathcal{M}_\mathcal{K}, \pi}$. Often, in absence of an exact knowledge of the transition matrix $P$, the policy $\pi$ may be executed on an estimated transition matrix $\hat{P}$. We shall use $d_{\hat{\mathcal{M}}_\mathcal{K},\pi}$ to denote the state distribution of the policy $\pi$ executed on the MDP with the transition matrix $\hat{P}$. Also, define the following.
\begin{align*}
P_{\mathcal{K}}(\text{escape}|\pi) &= \mathop{\mathbb{E}}_{\tau\sim P(\cdot|\pi)}  \mathbf{1}_{\exists t< t_0:s_t\not\in \mathcal{K}, \tau=(s_0, a_0, \dots)},\\
P_{\mathcal{K},\gamma}(\text{escape}|\pi) &= (1-\gamma) \mathop{\mathbb{E}}_{\tau\sim P(\cdot|\pi)} \sum_{t=0}^\infty \gamma^t \mathbf{1}_{\substack{s_u\in\mathcal{K} \forall u<t \text{ and }\\ s_t\not\in \mathcal{K}, \tau=(s_0, a_0, \dots)}}.
\end{align*}
Note that $P_{\mathcal{K}}(\text{escape}|\pi)\geq P_{\mathcal{K},\gamma}(\text{escape}|\pi)-\gamma^{t^0}$.

\begin{lemma}\label{k1}(Lemma 8.4.4\cite{kakade2003sample})
For any policy $\pi$, the following statements are valid.
\begin{align*}
\langle d_\pi, r\rangle &\geq \langle d_{\mathcal{M}_\mathcal{K},\pi}, r_\mathcal{K}\rangle - P_{\mathcal{K},\gamma}(\text{escape}|\pi)\|r_\mathcal{K}\|_\infty , \\
\langle d_{\mathcal{M}_\mathcal{K},\pi}, r_\mathcal{K}\rangle &\geq\langle d_\pi, r\rangle.
\end{align*}
\end{lemma}

\begin{lemma}\label{k2}(Lemma 8.5.4\cite{kakade2003sample})
If, for all $(s,a)\in\mathcal{S}\times \mathcal{A}$, $\|\hat{P}(\cdot|s,a)-P_\mathcal{K}(\cdot|s,a)\|_1\leq \varepsilon$, then for any reward $r$, policy $\pi$, it is true that
\[ |\langle d_{\mathcal{M}_\mathcal{K},\pi},r\rangle - \langle d_{\hat{\mathcal{M}}_\mathcal{K},\pi}, r\rangle | \leq \frac{\varepsilon}{1-\gamma}\]
\end{lemma}

\begin{lemma}\label{k3}(Folklore, eg. Lemma 8.5.5\cite{kakade2003sample})
When $m$ samples $\{x_1,\dots x_m\}$ are drawn from a distribution $P$, supported on a domain of size $d$, to construct an empirical distribution $\hat{P}(x) = \frac{\sum_{i=1}^m \mathbf{1}_{x_i=x}}{m}$, it is guaranteed that with probability $1-\delta$
\[ \|P-\hat{P}\|_1\leq \sqrt{\frac{8d\log\frac{2d}{\delta}}{m}}.\]
\end{lemma}

\begin{proof}[Proof of Lemma~\ref{lorcA}] 
The key observation in dealing with an unknown MDP is: either $\pi$, computed on the the transition $\hat{P}$, is (almost) optimal for the given reward signal on the true MDP, or it escapes the set of known states $\mathcal{K}$ quickly. If the former occurs, the requirement on the output of the algorithm is met. In case of the later, $\pi$ serves as a good policy to quickly explore new states -- this can happen only a finite number of times.

Let $\pi^*=\argmax_{\pi} V_\pi$. First, note that for any $\pi$ chosen in the Line~\ref{consPI}, we have
\ifdefined\arxiv
\begin{align*}
    V_\pi =& \langle d_{\pi}, r\rangle \\
    \geq& \langle d_{\mathcal{M}_\mathcal{K},\pi}, r_\mathcal{K}\rangle - (\gamma^{t_0} + P_{\mathcal{K}}(\text{escape}|\pi))B & \ref{k1}\\
    \geq& \langle d_{\hat{\mathcal{M}}_\mathcal{K},\pi}, r_\mathcal{K}\rangle - \frac{1}{1-\gamma}\sqrt{\frac{8|\mathcal{S}|\log \frac{2|\mathcal{S}|}{\delta}}{m}} - (\gamma^{t_0}+P_{\mathcal{K}}(\text{escape}|\pi))B & \ref{k2},\ref{k3}\\
    \geq& \langle d_{\hat{\mathcal{M}}_\mathcal{K},\pi^*}, r_\mathcal{K}\rangle - \varepsilon_1 - \frac{1}{1-\gamma}\sqrt{\frac{8|\mathcal{S}|\log \frac{2|\mathcal{S}|}{\delta}}{m}} -  (\gamma^{t_0}+P_{\mathcal{K}}(\text{escape}|\pi))B & \text{choice of } \pi\\   
    \geq& \langle d_{\mathcal{M}_\mathcal{K},\pi^*}, r_\mathcal{K}\rangle - \varepsilon_1 - \frac{2}{1-\gamma}\sqrt{\frac{8|\mathcal{S}|\log \frac{2|\mathcal{S}|}{\delta}}{m}} -  (\gamma^{t_0}+P_{\mathcal{K}}(\text{escape}|\pi))B & \ref{k2},\ref{k3}\\
    \geq& V_{\pi^*} - \varepsilon_1 - \frac{2}{1-\gamma}\sqrt{\frac{8|\mathcal{S}|\log \frac{2|\mathcal{S}|}{\delta}}{m}} -  (\gamma^{t_0}+P_{\mathcal{K}}(\text{escape}|\pi)) B & \ref{k1}
\end{align*}
\else
\begin{align*}
    &V_\pi = \langle d_{\pi}, r\rangle \\
    \geq& \langle d_{\mathcal{M}_\mathcal{K},\pi}, r_\mathcal{K}\rangle - (\gamma^{t_0} + P_{\mathcal{K}}(\text{escape}|\pi))B & \ref{k1}\\
    \geq& \langle d_{\hat{\mathcal{M}}_\mathcal{K},\pi}, r_\mathcal{K}\rangle - \frac{1}{1-\gamma}\sqrt{\frac{8|\mathcal{S}|\log \frac{2|\mathcal{S}|}{\delta}}{m}} \\ &- (\gamma^{t_0}+P_{\mathcal{K}}(\text{escape}|\pi))B & \ref{k2},\ref{k3}\\
    \geq& \langle d_{\hat{\mathcal{M}}_\mathcal{K},\pi^*}, r_\mathcal{K}\rangle - \varepsilon_1 - \frac{1}{1-\gamma}\sqrt{\frac{8|\mathcal{S}|\log \frac{2|\mathcal{S}|}{\delta}}{m}} \\ &-  (\gamma^{t_0}+P_{\mathcal{K}}(\text{escape}|\pi))B & \text{choice of } \pi\\   
    \geq& \langle d_{\mathcal{M}_\mathcal{K},\pi^*}, r_\mathcal{K}\rangle - \varepsilon_1 - \frac{2}{1-\gamma}\sqrt{\frac{8|\mathcal{S}|\log \frac{2|\mathcal{S}|}{\delta}}{m}} \\ &-  (\gamma^{t_0}+P_{\mathcal{K}}(\text{escape}|\pi))B & \ref{k2},\ref{k3}\\
    \geq& V_{\pi^*} - \varepsilon_1 - \frac{2}{1-\gamma}\sqrt{\frac{8|\mathcal{S}|\log \frac{2|\mathcal{S}|}{\delta}}{m}} \\ &-  (\gamma^{t_0}+P_{\mathcal{K}}(\text{escape}|\pi)) B & \ref{k1}
\end{align*}
\fi

If $P_{\mathcal{K}}(\text{escape}|\pi)>\Delta$, then the probability that $\pi$ doesn't escape $\mathcal{K}$ in $n$ trials is $e^{-n\Delta}$. Accounting for the failure probabilities with a suitable union bound, Line~\ref{cert} ensures that $\pi$ is marked {\em stable} only if $P_{\mathcal{K}}(\text{escape}|\pi)\leq \frac{\log (N\delta^{-1})}{n}$, where $N$ is the total number of times the inner loop is executed.

To observe the truth of the second part of the claim, note that every reiteration of the inner loop coincides with the exploration of some action at a $m$-\emph{unknown} state. There can be at most $m|\mathcal{S}||\mathcal{A}|$ such \emph{exploration} steps. Finally, each run of the inner loop samples $n$ episodes.
\end{proof}

\begin{proof}[Proof of Lemma~\ref{lorcB}.]
First note that it suffices to ensure for all $t<t_0$ simultaneously, it happens $\|d_{t,\pi}-\hat{d}_{t,\pi}\|_\infty \leq 0.1\varepsilon_0$. This is because
\ifdefined\arxiv
\begin{align*}
    \|d_\pi-\hat{d}_\pi\|_\infty &\leq \frac{1-\gamma}{(1-\gamma^{t_0})}\sum_{t=0}^{t_0-1} \gamma^t \|\hat{d}_{t,\pi}-(1-\gamma^{t_0})d_{t,\pi}\|_\infty + \gamma^{t_0} \\
    &\leq \frac{1-\gamma}{(1-\gamma^{t_0})}\sum_{t=0}^{t_0-1}\gamma^t \|\hat{d}_{t,\pi}-d_{t,\pi}\| + 0.3\varepsilon_0 \leq \varepsilon_0.
\end{align*}
\else
\begin{align*}
    &\|d_\pi-\hat{d}_\pi\|_\infty \\
    &\leq \frac{1-\gamma}{(1-\gamma^{t_0})}\sum_{t=0}^{t_0-1} \gamma^t \|\hat{d}_{t,\pi}-(1-\gamma^{t_0})d_{t,\pi}\|_\infty + \gamma^{t_0} \\
    &\leq \frac{1-\gamma}{(1-\gamma^{t_0})}\sum_{t=0}^{t_0-1}\gamma^t \|\hat{d}_{t,\pi}-d_{t,\pi}\| + 0.3\varepsilon_0 \leq \varepsilon_0.
\end{align*}
\fi
Since the trajectories are independently, $|\hat{d}_{t,\pi}(s)-d_{t,\pi}(s)| \leq \sqrt{\frac{2}{m}\log\frac{2}{\delta}}$ for each $t<t_0$ and state $s\in\mathcal{S}$ with probability $1-\delta$, by Hoeffding's inequality. A union bound over states and $t$ concludes the proof. 
\end{proof}

\section{Proof of Concept Experiments}  \label{sec:experiments}
\begin{figure*}[t!]
\centering
    \setlength\tabcolsep{0.5pt}
    \begin{tabular}{c c c}
    \begin{minipage}{0.33\linewidth}\center
        \texttt{MountainCar}
    \end{minipage} &
    \begin{minipage}{0.33\linewidth}\center
        \texttt{Pendulum}
    \end{minipage} &\begin{minipage}{0.05\linewidth}
        \texttt{Ant}
    \end{minipage}   \\\vspace{-0.5cm}
    \begin{minipage}{0.33\linewidth}
        \includegraphics[width=\textwidth]{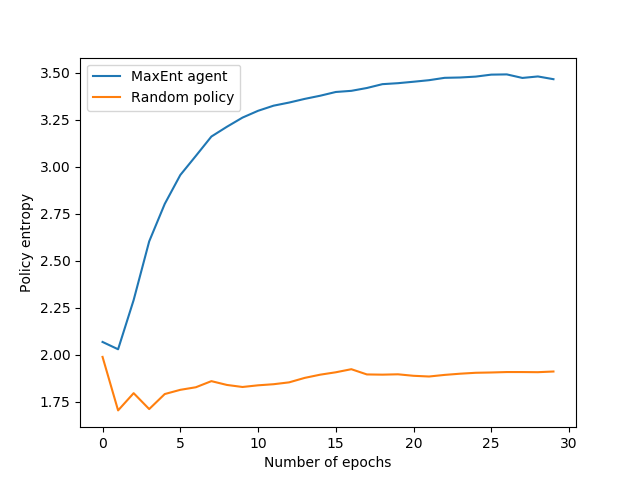}
        \subcaption{}\label{fig:exp1-mc}
    \end{minipage} &
    \begin{minipage}[]{0.33\linewidth}
      \includegraphics[width=\textwidth]{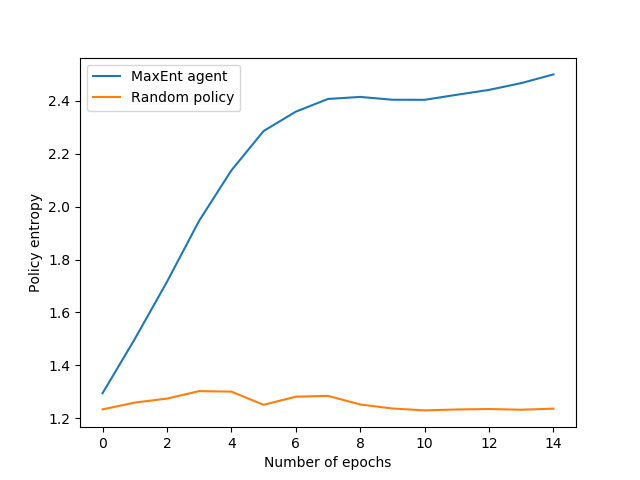}
      \subcaption{}\label{fig:exp1-pend}
    \end{minipage} &
    \begin{minipage}{0.33\linewidth}
        \includegraphics[width=\textwidth]{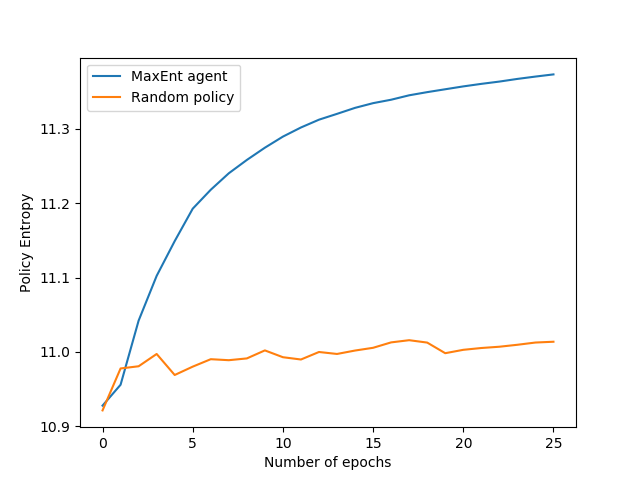}
        \subcaption{}\label{fig:exp1-ant}
    \end{minipage} \\\vspace{-0.2cm}
    \begin{minipage}{0.3\linewidth}
        \includegraphics[width=\textwidth]{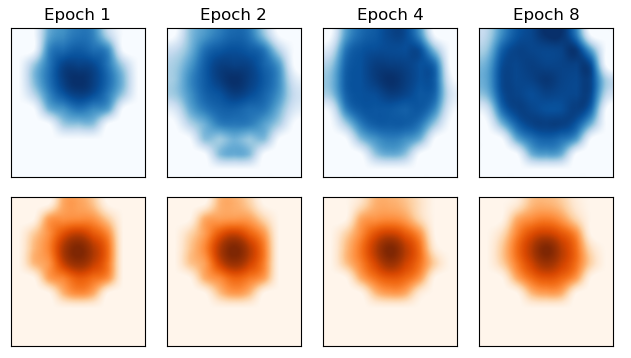}
        \subcaption{}\label{fig:exp2-mc}
    \end{minipage} &
    \begin{minipage}[]{0.32\linewidth}
      \includegraphics[width=\textwidth]{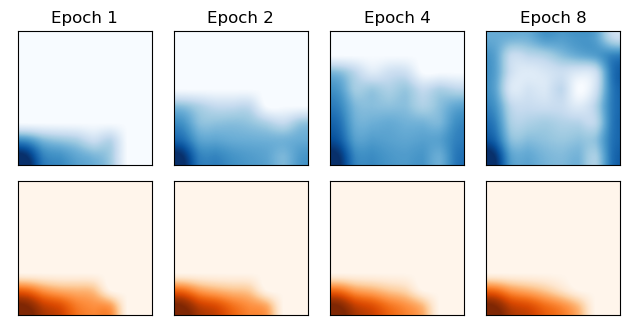}
      \subcaption{}\label{fig:exp2-pend}
    \end{minipage} &
    \begin{minipage}{0.3\linewidth}
        \includegraphics[width=\textwidth]{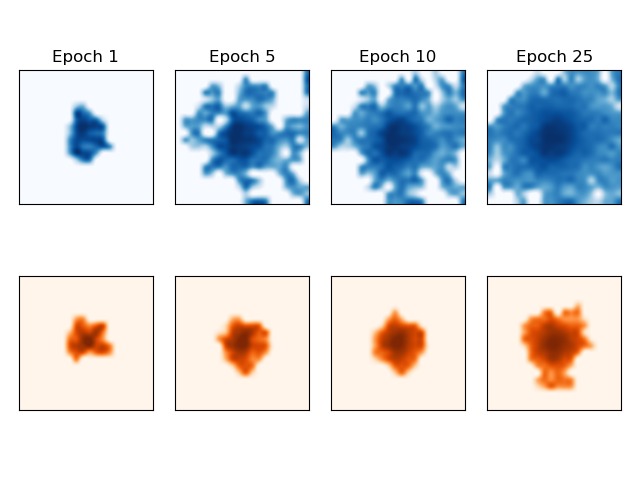}
        \subcaption{}\label{fig:exp2-ant}
    \end{minipage}
    \end{tabular}
\caption{Results of the preliminary experiments. In each plot, {\color{blue}blue} represents the MaxEnt agent, and {\color{orange}orange} represents the random baseline. \ref{fig:exp1-mc}, \ref{fig:exp1-pend}, and \ref{fig:exp1-ant} show the entropy of the policy evolving with the number of epochs. \ref{fig:exp2-mc}, \ref{fig:exp2-pend}, and \ref{fig:exp2-ant} show the log-probability of occupancy of the two-dimensional state space. In \ref{fig:exp2-ant}, the infinite $xy$ grid is limited to $[-20, 20]\times[-20,20]$.}\label{fig:exps}
\end{figure*}

\begin{figure}[t!]
    \centering
    \begin{subfigure}[b]{0.5\textwidth}
     \includegraphics[width=\textwidth]{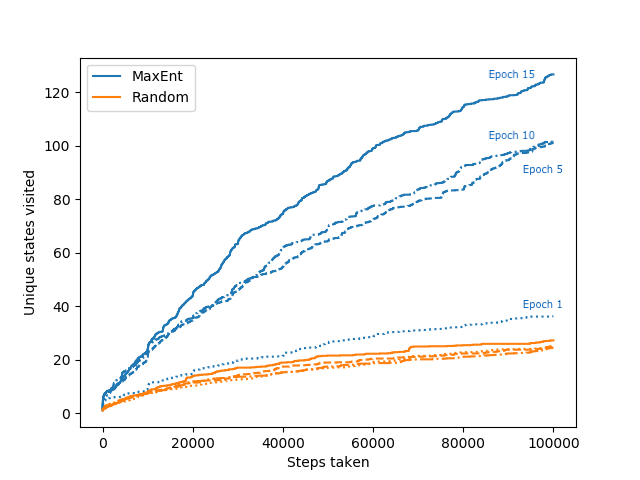}
    \end{subfigure}
    \caption{The number of distinct $xy$ states visited by \texttt{Ant} at various epochs. Results were averaged over $N=20$ executions. As the number of policies in the mixture increases, the agent reaches more unique states in the same amount of time.}\label{fig:states-visited}
\end{figure}

We report the results from a preliminary set of experiments\footnote{The open-source implementations may be found at \url{https://github.com/abbyvansoest/maxent_base} and \url{https://github.com/abbyvansoest/maxent_ant}.}. In each case, the MaxEnt agent learns to access the set of reachable states within a small number of iterations, while monotonically increasing the entropy of the induced state distribution. 

Recall that Algorithm~\ref{mainA} requires access to an approximate planning oracle and a density estimator for the induced distribution. Here the density estimator is deliberately chosen to be simple -- a count-based estimate over the discretized state space. It is possible to use neural density estimators and other function-approximation based estimators in its stead.

\subsection{Environments and the discretization procedure}

The $2$-dimensional state spaces for \texttt{MountainCar} and \texttt{Pendulum} (from \cite{brockman2016openai}) were discretized evenly to grids of size $10\times 9$ and $8\times 8$, respectively. For \texttt{Pendulum}, the maximum torque and velocity were capped at $1.0$ and $7.0$, respectively. 

 The $29$-dimensional state space for \texttt{Ant} (with a Mujoco engine) was first reduced to dimension $7$, combining the agent's $x$ and $y$ location in the gridspace with a $5$-dimensional random projection of the remaining $27$ states. The $x$ and $y$ dimensions were discretized into $16$ bins in the range $[-12, 12]$. The other dimensions were each normalized and discretized into $15$ bins in the range $[-1, 1]$. While the planning agent agent had access to the full state representation, the density estimation was performed exclusively on the reduced representation.

\subsection{Algorithmic details}

\textit{Reward function.} Each planning agent was trained to maximize a \textrm{KL} divergence objective function, ie. $KL(\textrm{Unif} || d_\pi)$. 

\textit{MountainCar and Pendulum}. The planning oracle is a REINFORCE \cite{sutton2000policy} agent, where the the output policy from the previous iteration is used as the initial policy for the next iteration. The policy class is a neural net with a single hidden layer consisting of 128 units. The agent is trained on 400 and 200 episodes every epoch for \texttt{MountainCar} and \texttt{Pendulum}, respectively. The baseline agent chooses its action randomly at every time step.

\textit{Ant}. The planning oracle is a Soft Actor-Critic \cite{DBLP:journals/corr/abs-1801-01290} agent. The policy class is a neural net with 2 hidden layers composed of 300 units and the ReLU activation function. The agent is trained for 30 episodes, each of which consists of a roll-out of 5000 steps. The mixed policy is executed over $10$ trials of $T=10000$ steps at the end of each epoch in order to approximate the policy distribution and compute the next reward function. The baseline agent chooses its actions randomly for the same number of trials and steps.

\subsection*{Acknowledgements}
Sham Kakade acknowledges funding from the Washington Research Foundation for Innovation in Data-intensive Discovery, the DARPA award FA8650-18-2-7836, and the ONR award N00014-18-1-2247. The authors thank Shie Mannor for helpful discussions.

\bibliography{main,ref}
\bibliographystyle{alpha}

\end{document}